\documentclass[links,compactfloats,tecrep,palatino,fullpage]{paper}

\titlecontent{
Concentration and Confidence for Discrete Bayesian Sequence Predictors 
}

\authors{{\bf Tor Lattimore} and {\bf Marcus Hutter} and {\bf Peter Sunehag}}

\abstractcontent{
Bayesian sequence prediction is a simple technique for predicting future symbols sampled from an unknown measure on infinite sequences
over a countable alphabet. While strong bounds on the expected cumulative error are known, there are only limited results on the distribution
of this error. We prove tight high-probability bounds on the cumulative error, which is measured in terms of the Kullback-Leibler 
(KL) divergence.
We also consider the problem of constructing upper confidence bounds on the KL and Hellinger errors similar to those constructed 
from Hoeffding-like bounds in the i.i.d.\ case.
The new results are applied to show that Bayesian sequence prediction can be used in the Knows What It Knows (KWIK) framework with bounds that match the state-of-the-art.
}

\keywordscontent{
Bayesian sequence prediction;
concentration of measure;
information theory;
KWIK learning.
}

\begin{document}

\buildtitle

%%%%%%%%%%%%%%%%%%%%%%%%%%%%%%%%%%%%%%%%%%%%%%%%%%%%%%%%%%%%%%%
%% INTRODUCTION
%%%%%%%%%%%%%%%%%%%%%%%%%%%%%%%%%%%%%%%%%%%%%%%%%%%%%%%%%%%%%%%
\section{Introduction}

Sequence prediction is the task of predicting symbol $\omega_t$ having observed $\omega_{<t} = \omega_1 \omega_2 \omega_3 \cdots \omega_{t-1}$ where the underlying
distribution from which the sequence is sampled is unknown and may be non-stationary. We assume sequences are sampled from 
an unknown measure $\mu$ known to be contained in a countable model class $\M$. At time-step $t$ having observed $\omega_{<t}$ a predictor 
$\rho$ should output a distribution $\rho_t$ over the next symbol $\omega_t$. A predictor may be considered good if for all $\mu \in \M$ 
the predictive distribution of $\rho$ converges fast to that of $\mu$
\eq{
\Delta(\rho_t, \mu_t) \stackrel{fast}\longrightarrow 0
}
where $\Delta(\rho_t, \mu_t)$ is some measure of the distance between $\rho_t$ and $\mu_t$, typically either the Kullback-Leibler (KL) divergence $d_t$ or the squared Hellinger distance $h_t$.
One such predictor is the Bayesian mixture $\xi$ over all $\nu \in \M$ with strictly positive prior.
A great deal is already known about $\xi$. In particular the predictive distribution $\xi_t$ converges to $\mu_t$ with $\mu$-probability one and
does so with finite expected cumulative error with respect to both the KL divergence and
the squared Hellinger distance \cite{BD62,Sol78,Hut01b,Hut03,Hut05}.

The paper is divided into three sections. In the first we review the main results bounding the expected cumulative error
between $\mu_t$ and $\xi_t$ and prove high-probability bounds on this quantity. 
Such bounds are already known for the squared Hellinger distance, but not the KL divergence until now \cite{HM07}. We also bound the cumulative $\xi$-expected information
gain.
The second section relates to the confidence of the Bayes predictor. Even though $h_t$ and $d_t$ converge fast to zero,
these quantities cannot be computed without knowing $\mu$. We construct 
confidence bounds $\hat h_t$ and $\hat d_t$ that are computable from the observations and upper bound $h_t$ and $d_t$ with high probability respectively.
Furthermore we show that $\hat h_t$ and $\hat d_t$ also converge fast to zero and so can be used in the place of the unknown $h_t$ and $d_t$.
The results serve a similar purpose to upper confidence bounds obtained from Hoeffding-like bounds in the i.i.d.\ case to which our bounds are roughly comparable
\iftecrep
(Appendix \ref{app:experiments}).
\else
(\cite{LHS13bayes-conc-tech}).
\fi
Finally we present a simple application of the new results by showing that Bayesian sequence prediction
can be applied to the Knows What It Knows (KWIK) framework \cite{LLWS11} where we achieve a state-of-the-art bound using a simple, efficient and principled algorithm.

%%%%%%%%%%%%%%%%%%%%%%%%%%%%%%%%%%%%%%%%%%%%%%%%%%%%%%%%%%%%%%%
%% NOTATION
%%%%%%%%%%%%%%%%%%%%%%%%%%%%%%%%%%%%%%%%%%%%%%%%%%%%%%%%%%%%%%%
\section{Notation}

\iftecrep
A table summarising the notation presented in this section may be found in Appendix \ref{app:not}.
\fi
The natural numbers are denoted by $\N$. Logarithms are taken with respect to base $e$. The indicator function is $\ind{expr}$, which is equal to $1$ if
$expr$ is true and $0$ otherwise.
The alphabet $\A$ is a finite or countable set of symbols.
A finite string $x$ over alphabet $\A$ is a sequence
$x_1 x_2 x_3 \cdots x_n$ where $x_k \in \A$. An infinite string is a sequence
$\omega_1 \omega_2 \omega_3 \cdots$. We denote the set of all finite strings by $\A^*$ and the set of infinite strings
by $\A^\infty$. The length of finite string $x \in \A^*$ is denoted by $\ell(x)$. Strings can be concatenated.
If $x \in \A^*$ and $y \in \A^* \cup \A^\infty$, then $xy$ is the concatenation of $x$ and $y$. For string $x \in \A^* \cup \A^\infty$, 
substrings are denoted by $x_{1:t} = x_1 x_2 \cdots x_t$ and $x_{<t} = x_{1:t-1}$.
The empty string of length zero is denoted by $\epsilon$.

\subsubsect{Measures}
The cylinder set of finite string $x$ is $\Gamma_x \defined \set{x \omega : \omega \in \A^\infty}$. Define
$\sigma$-algebra $\F_{<t} \defined \sigma(\set{\Gamma_x : x \in \A^{t-1}})$ and $\F \defined \sigma(\set{\Gamma_x : x \in \A^*})$.
Then $(\A^\infty, \set{\F_{<t}}, \F)$ is a filtered probability space. Let $\mu$ be a probability measure on this space. We abuse notation by using the shorthands
$\mu(x) \defined \mu(\Gamma_x)$ and
$\mu(y | x)  \defined {\mu(xy) / \mu(x)} $.
The intuition is that $\mu(x)$ represents the $\mu$-probability that an infinite sequence sampled from $\mu$
starts with $x$ and $\mu(y|x)$ is the $\mu$-probability that an infinite sequence sampled from $\mu$ starts with $xy$ given that it starts with $x$.
We write $\mu \absolute \xi$ if $\mu$ is absolutely continuous with respect to $\xi$. From now on, unless otherwise specified, all measures will be probability measures
on filtered probability space $(\A^\infty, \set{\F_{<t}}, \F)$. 

\subsubsect{Bayes mixture}
Let $\M$ be a countable set of measures and $w : \M \to (0,1]$ be a probability distribution
on $\M$. The Bayes mixture measure $\xi:\F \to [0,1]$ is defined by
$\xi(A) \defined \sum_{\nu \in \M} w_\nu \nu(A)$.
By the definition $\xi(A) \geq w_\nu \nu(A)$ for all $A \in \F$ and $\nu \in \M$, which implies that $\nu \absolute \xi$.
Having observed data $x \in \A^*$ the prior $w$ is updated using Bayes rule to be
$w_\nu(x) \defined w_\nu {\nu(x) / \xi(x)}$.
Then $\xi(y|x)$ can be written
$\xi(y|x) = \sum_{\nu \in \M} w_\nu(x) \nu(y|x)$.
The entropy of the prior is $\ent(w) \defined -\sum_{\nu \in \M} w_\nu \ln w_\nu$.

\subsubsect{Distances between measures}
Let $\mu$ and $\xi$ be measures.
The squared Hellinger distance between the predictive distributions of $\mu$ and $\xi$ given $x \in \A^*$ is defined by
$\hellinger{x}{\mu}{\xi} \defined \sum_{a \in \A} (\sqrt{\mu(a|x)} - \sqrt{\xi(a|x)})^2$.
If $\mu \absolute \xi$, then the Kullback-Leibler (KL) divergence is defined by
$\KL{x}{\mu}{\xi} \defined \sum_{a \in \A} \mu(a|x) \ln {\mu(a|x) \over \xi(a|x)}$.
The KL divergence is not a metric because it satisfies neither the symmetry nor the triangle inequality properties. Nevertheless, it 
is a useful measure of the difference between measures and is occasionally more convenient than the Hellinger distance. 
Let $\xi$ be the Bayes mixture over $\nu \in \M$ with prior $w : \M \to (0,1]$. 
If $\rho \in \M$, then define random variables on $X^\infty$ by 
\eq{
\rho_{1:t}(\omega) &\defined \rho(\omega_{1:t}) & \rho_{<t}(\omega) &\defined \rho(\omega_{<t}) & \rho_{t}(\omega) &\defined \rho(\omega_t|\omega_{<t}) 
}
\eq{
\hellinger{t}{\rho}{\xi}(\omega) &\defined \hellinger{\omega_{<t}}{\rho}{\xi} &
\KL{t}{\rho}{\xi}(\omega) &\defined \KL{\omega_{<t}}{\rho}{\xi} 
}
The latter term can be rewritten as 
\eqn{
\label{eq:KL}
\KL{t}{\rho}{\xi} = \E_\rho\left[\ln {\rho_{1:t} \over \rho_{<t}} \cdot {\xi_{<t} \over \xi_{1:t}}\bigg|\F_{<t}\right]
= \E_\rho\left[\ln{\rho_{1:t} \over \xi_{1:t}}\bigg|\F_{<t}\right] + \ln {\xi_{<t} \over \rho_{<t}}.
}
Now fix an unknown $\mu \in \M$ and define random variables (also on $X^\infty$).
\eq{
d_t &\defined \KL{t}{\mu}{\xi} & 
h_t &\defined \hellinger{t}{\mu}{\xi} &
c_t(\omega)&\defined \sum_{\nu \in \M} w_\nu(\omega_{<t}) \KL{\omega_{<t}}{\nu}{\xi} \\ 
D_\infty &\defined \sum_{t=1}^\infty d_t 
&H_\infty &\defined \sum_{t=1}^\infty h_t 
&C_\infty &\defined \sum_{t=1}^\infty c_t.
}
Both $h_t$ and $d_t$ are well-known ``distances'' between the predictive distributions of $\xi$ and $\mu$ at time $t$. 
The other quantity $c_t$ is the 
$\xi$-expected information gain of the posterior between times $t$ and $t+1$ given the observed sequence at time $t$. 
\eq{
c_t &= \sum_{\nu \in \M} w_\nu {\nu_{<t} \over \xi_{<t}} \KL{t}{\nu}{\xi} 
= \E_\xi\Bigg[\underbrace{\sum_{\nu \in \M} w_\nu {\nu_{1:t} \over \xi_{1:t}} \ln {\nu_t \over \xi_t}}_{\mathclap{\text{information gain}}}\Bigg|\F_{<t}\Bigg]
}
An important 
observation is that $c_t$ is independent of the unknown
$\mu$.

%%%%%%%%%%%%%%%%%%%%%%%%%%%%%%%%%%%%%%%%%%%%%%%%%%%%%%%%%%%%%%%
%% CONVERGENCE
%%%%%%%%%%%%%%%%%%%%%%%%%%%%%%%%%%%%%%%%%%%%%%%%%%%%%%%%%%%%%%%
\section{Convergence}

In this section we consider the convergence of $\xi_t - \mu_t \to 0$ for all $\mu \in \M$ where convergence holds with $\mu$-probability 1, in mean sum or with
high $\mu$-probability of a small cumulative error.
The first theorem is a version of the celebrated result of Solomonoff that the predictive distribution of the Bayes mixture $\xi$ converges fast to the truth in
expectation \cite{Sol78,Hut05}. The only modification is the alphabet is now permitted to be countable rather than finite. 

\begin{theorem}[\citen{Sol78,Hut05}]\label{thm:solomonoff-stopping}
The following hold:
\eq{
\E_\mu H_\infty \leq \E_\mu D_\infty &\leq \ln{1 \over w_\mu} &
\lim_{t\to\infty} d_t = \lim_{t\to\infty} h_t = 0,\quad w.\mu.p.1.
}
\end{theorem}
\iftecrep
The proof can be found in Appendix \ref{app:solomonoff}.
\else
The proof can be found in the extended technical report \cite{LHS13bayes-conc-tech}.
\fi
Theorem \ref{thm:solomonoff-stopping} shows that the predictive distribution of $\xi$ converges to $\mu$ asymptotically and that it does so fast (with finite cumulative squared Hellinger/KL error) in expectation.
We now move on to the question of high-probability bounds on $D_\infty$ and $H_\infty$. The following theorem is already known and essentially unimprovable.
\begin{theorem}[\citen{HM07}]\label{thm:hellinger-concentration}
For all $\delta \in (0,1)$ it holds with $\mu$-probability at least $1 - \delta$ that
$H_\infty \leq \ln\frac{1}{w_\mu} + 2\ln\frac{1}{\delta}$.
\end{theorem}
We contribute a comparable concentration bound for $D_\infty$. 
A weak bound can be obtained by applying Markov's inequality to show that
$D_\infty \leq \frac{1}{\delta} \cdot \ln (\frac{1}{w_\mu}) $ with $\mu$-probability at
least $1-\delta$, but a stronger result is possible.
\begin{theorem}\label{thm:probability-bound}
For all $\delta \in (0, 1)$ it holds with $\mu$-probability at least $1 - \delta$ that
$D_\infty \leq e \cdot (\ln\frac{6}{\delta}) \cdot (\ln\frac{2}{\delta} + \ln \frac{1}{w_\mu} )$.
\end{theorem}

\begin{proof}
A stopping time is a random variable $t : \A^\infty \to \N \cup \set{\infty}$ such that $t^{-1}(n)$ is $\F_{<n}$ measurable
for all $n$. For stopping time $t$ let $X(t) \subset \A^*$ be the set of finite sequences where $t$ becomes known
\eq{
X(t) \defined \set{x : t(x\omega) = \ell(x) + 1, \forall \omega}.
}
Define random variable $z_{<t} \defined \xi_{<t} / \mu_{<t}$ and $L \defined \ceil{\ln (2/\delta)} \leq \ln (6 / \delta)$ 
and stopping times $\set{t_k}$ inductively by 
\eq{
t_1 &\defined 1 & t_{k+1} &\defined \min\set {s : \textsum_{t=t_{k}}^{s} d_t > e \cdot \left(\ln z_{<t_k} + \ln{1 \over w_\mu}\right)}.
}
The result follows from two claims, which are proven later. 
\begin{center}
\begin{tikzpicture}
\node[draw,double,minimum width=6.0cm,minimum height=1.5cm,text width=6cm,align=left] at (0,0) 
{$\displaystyle\P{\sup_t \ln z_{<t} \geq {\ln {2 \over \delta}}} \leq \delta / 2$\hfill($\star$)};
\node[draw,double,minimum width=5.0cm,minimum height=1.5cm,text width=5cm,align=left] at (6.3,0) 
{$\displaystyle{\mu\bigg(t_{L+1} < \infty\bigg) \leq \delta / 2}$\hfill($\star\star$)};
\end{tikzpicture} 
\end{center}
By the union bound we obtain that if $A$ is the event that $t_{L+1} = \infty$ and $\sup_t \ln z_{<t} \leq \ln{2 \over \delta}$, then $\mu(A) \geq 1 - \delta$
and for $\omega \in A$
\eq{
D_\infty(\omega) &= \sum_{t=1}^\infty d_t(\omega) 
\sr{(a)}= \sum_{k=1}^L \sum_{t=t_k(\omega)}^{t_{k+1}(\omega)-1} d_t(\omega) 
\sr{(b)}\leq \sum_{k=1}^{L}  e\cdot \left(\ln z_{<t_k}(\omega) + \ln {1 \over w_\mu}\right) \\ 
&\sr{(c)}\leq e\cdot L\left(\ln{2 \over \delta} +  \ln {1 \over w_\mu}\right) 
\sr{(d)}\leq e \cdot \ln \left({6 \over \delta}\right) \cdot \left(\ln{2 \over \delta} + \ln{1 \over w_\mu}\right)
}
where (a) follows from the definition of $t_k$ and because $t_{L+1}(\omega) = \infty$. 
(b) follows from the definition of $t_k$.
(c) because $\sup_t \ln z_{<t} \leq \ln \frac{2}{\delta}$.
(d) by the definition of $L$.
The theorem is completed by proving $(\star)$ 
and $(\star\star)$.
The first follows immediately from Lemma \ref{lem:supermartingale}. 
For the second we use induction and
Theorem \ref{thm:solomonoff-stopping}. 
After observing $x \in \A(t_n)$, $\xi(\cdot|x)$ is a Bayes mixture over $\nu(\cdot|x)$ where $\nu \in \M$ with prior weight 
$w(\nu(\cdot|x)) = w_\nu \nu(x) / \xi(x)$. Therefore by Theorem \ref{thm:solomonoff-stopping}
\eq{
\E_\mu \left[ \textsum_{t=\ell(x)+1}^\infty d_t \Bigg| x\right] \leq \ln {1 \over w(\mu(\cdot|x))} =\ln {\xi(x) \over \mu(x)} + \ln {1 \over w_\mu}.
}
Therefore by Markov's inequality
\eq{
\P{\textsum_{t=\ell(x)+1}^\infty d_t > e \cdot \left(\ln {\xi(x) \over \mu(x)} + \ln {1 \over w_\mu}\right)\Bigg| x} \leq {1 \over e}.
}
Let $n \in \N$ and assume $\mu(t_{n} < \infty) \leq e^{1-n}$.
By the definition $t_{n+1} \geq t_n$ we have 
\eq{
\mu\bigg(t_{n+1} < \infty \bigg) 
&=\sum_{\mathclap{x \in \A(t_n)}} \mu(x) \cdot \mu\left(\textsum_{t=\ell(x)+1}^\infty d_t 
> e \cdot \left(\ln {\xi(x) \over \mu(x)} + \ln {1 \over w_\mu}\right)\Bigg|x \right) \\
&\leq {1 \over e} \sum_{x \in \A(t_n)} \mu(x) = {1 \over e} \mu(t_n < \infty) \leq e^{-n}.
}
Therefore $\mu(t_n < \infty) \leq e^{1-n}$ for all $n$ and so $\mu(t_{L+1} < \infty) \leq e^{-L} \leq \delta/2$, which completes
the proof of $(\star\star)$ and so also the theorem.
\end{proof}
Theorem \ref{thm:probability-bound} is close to unimprovable.
\begin{proposition}\label{prop:probability-lower-bound}
There exists an $\M = \set{\mu, \nu}$ such that with
$\mu$-probability at least $\delta$ it holds that
$D_\infty > {1 \over 4 \ln 2} {\ln \frac{1}{\delta} } \left({\ln \frac{1}{\delta} } + 2\ln {1 - w \over w} - 3 \ln 2\right)$.
\end{proposition}

\begin{proof}
Let $\A = \set{0,1}$ and $\M \defined \set{\mu, \nu}$ where the true measure $\mu$ is the Lebesgue measure and $\nu$ is the measure deterministically producing
an infinite sequence of ones, which are defined by
$\mu(x) \defined 2^{-\ell(x)}$ and 
$\nu(x) \defined \ind{x = 1^{\ell(x)}}$ where $1^n$ is the sequence of $n$ ones..
Let $w = w_\mu$ and $w_\nu = 1 - w$. If $n = \floor{{1 \over \ln 2} \ln {1 \over \delta}} \in \N$, then $\mu(\Gamma_{1^n}) \geq \delta$ and
for $\omega \in \Gamma_{1^n}$
\eq{
D_{\infty}&(\omega) 
\sr{(a)}\geq \sum_{t=1}^{n+1} \KL{1^{t-1}}{\mu}{\xi} 
\sr{(b)}= \sum_{t=1}^{n+1} \left({1 \over 2} \cdot \ln {\frac{1}{2} \over \xi(1|1^{t-1})} + {1 \over 2}\cdot \ln {\frac{1}{2} \over \xi(0|1^{t-1})} \right) \\
&\!\sr{(c)}> {1 \over 2} \sum_{t=1}^{n+1} \ln\left({1 \over {4\xi(0|1^{t-1})}}\right)  
\sr{(d)}= {1 \over 2} \sum_{t=1}^{n+1} \ln\left({w \cdot 2^{1-t} + (1 - w) \over 4w \cdot 2^{-t}}\right)  \\
&\!\sr{(e)}\geq {1 \over 2} \sum_{t=1}^{n+1} \left((t - 2) \ln 2 + \ln{1 - w \over w} \right) 
\sr{(f)}= {(n+1)\left(2\ln{1 - w \over w} + (n - 2) \ln 2\right) \over 4}
}
(a) follows from the definition of $D_\infty(\omega)$ and the positivity of the KL divergence, which allows the sum to be truncated.
(b) follows by inserting the definitions of $\mu$ and the KL divergence.
(c) by basic algebra and the fact that $\xi(1|1^{t-1}) < 1$. (d) follows from the definition of $\xi$ while (e) and (f) are basic
algebra. Finally substitute $n + 1 \geq {1 \over \ln 2} \ln {1 \over \delta}$.
\end{proof}

In the next section we will bound $d_t$ by a function of $c_t$, which can be computed without knowing $\mu$.
For this result to be useful we need to show that $c_t$ converges to zero, which is established by the following
theorems.
\begin{theorem}\label{thm:curiosity}
If $\ent(w) < \infty$, then $\E_{\mu} C_\infty \leq \ent(w) / w_\mu$ and
$\lim_{t\to\infty} c_t = 0$ with $\mu$-probability $1$.
\end{theorem}

\begin{proof}
We make use of the dominance $\xi(x) \geq w_\mu \mu(x)$, properties of expectation and Theorem \ref{thm:solomonoff-stopping}.
\eq{
\E_\mu C_\infty 
&\defined \E_\mu \sum_{t=1}^\infty c_t 
\sr{(a)}\leq {1 \over w_\mu} \E_\xi \sum_{t=1}^\infty c_t 
\sr{(b)}= {1 \over w_\mu} \E_\xi \sum_{t=1}^\infty \sum_{\nu \in \M} w_\nu {\nu_{<t} \over \xi_{<t}} \KL{t}{\nu}{\xi} \\
&\sr{(c)}= {1 \over w_\mu} \sum_{\nu \in \M}w_\nu  \E_\nu \sum_{k=1}^\infty \KL{t}{\nu}{\xi} 
\sr{(d)}\leq {1 \over w_\mu} \sum_{\nu \in \M}w_\nu  \ln{1 \over w_\nu}
\sr{(e)}={\ent(w) \over w_\mu}
}
(a) follows by dominance $\mu(A) \leq \xi(A) / w_\mu$ and linearity of expectation.
(b) is the definition of $c_t$.
(c) by exchanging sums and the definition of expectation.
(d) is true by substituting the result in Theorem \ref{thm:solomonoff-stopping}. 
Finally (e) follows from the definition of the entropy $\ent(w)$. 
That $\lim_{t\to\infty} c_t = 0$ with $\mu$-probability $1$ follows from the first result by applying Markov's inequality to bound $C_\infty < \infty$ with
probability $1$.
\end{proof}
In the finite case a stronger result is possible.
\begin{theorem}\label{thm:curiosity-finite}
If $|\M| = K < \infty$ and $w$ is the uniform prior, then
$\E_\mu C_\infty \leq  6 \ln^2 K + 14 \ln K + 8$.
\end{theorem}
Theorem \ref{thm:curiosity-finite} is tight in the following sense.
\begin{proposition}\label{prop:curiosity-lower}
For each $K \in \N$ there exists an $\M$ of size $K$ and $\mu \in \M$ such that if $w$ is the uniform prior on $\M$, then
$\E_\mu C_\infty > {1 \over 2} \ln^2 K - 1$.
\end{proposition}
See the appendix for the proofs of Theorem \ref{thm:curiosity-finite} and Proposition \ref{prop:curiosity-lower}.

%%%%%%%%%%%%%%%%%%%%%%%%%%%%%%%%%%%%%%%%%%%%%%%%%%%%%%%%%%%%%%%
%% CONFIDENCE
%%%%%%%%%%%%%%%%%%%%%%%%%%%%%%%%%%%%%%%%%%%%%%%%%%%%%%%%%%%%%%%
\section{Confidence}
In the previous section we showed that $\xi_t$ converges fast to $\mu_t$. 
One disadvantage of these results is that errors $d_t$ and $h_t$ cannot be determined without knowing $\mu$.
In this section we define $\hat d_t$ and $\hat h_t$ that 
upper bound $d_t$ and $h_t$ respectively with high probability and may be computed without knowing $\mu$.
Let $\M \supseteq \M_1 \supseteq \M_2 \cdots$ be a narrowing sequence of hypothesis classes where $\M_t$ contains the set of plausible models at time-step $t$ and is defined by
\eq{
\M_t \defined \set{\nu \in \M : \forall \tau \leq t, {\nu_{<\tau} \over \xi_{<\tau}} \geq \delta {w_\mu \over w_\nu}}
}
Then $\hat h_t$ is defined as the value maximising the weighted squared Hellinger distance between $\nu$ and $\xi$
for all plausible $\nu \in \M_t$ and $\hat d_t$ is defined in terms of the expected information gain.
\begin{center}
\begin{tikzpicture}
\node[draw,double,minimum width=4cm,minimum height=1.3cm,anchor=west] (d) at (0,0)  
{$\displaystyle \hat d_t \defined {c_t \over w_\mu \delta} $};
\node[draw,double,minimum width=5cm,minimum height=1.3cm,anchor=west] (h) at (6.4,0) 
{$\displaystyle \hat h_t \defined \sup_{\nu\in\M_t} \set{{w_\nu \over w_\mu }\hellinger{t}{\nu}{\xi}}$};
\end{tikzpicture}
\end{center}
Both $d_t$ and $h_t$ depend on $w_\mu$, which is also typically unknown. If $\M$ is finite, then the problem is easily side-stepped by choosing $w$ to be uniform.
The countable case is discussed briefly in the conclusion. 
First we prove that $h_t \leq \hat h_t$ and $d_t \leq \hat d_t$ with high probability after which we demonstrate that they are non-vacuous by proving that $\hat h_t$ 
and $\hat d_t$ converge fast to zero with high probability.
Now is a good time to remark that hypothesis testing using the factor $\nu_{<t} / \xi_{<t}$ is not exactly a new idea.
For discussion, results, history and references see \cite{SSVV11}.

\begin{theorem}\label{thm:confidence}
For all $\delta \in [0,1]$ it holds that:
\eq{
\mu(\forall t : d_t \leq \hat d_t) \geq 1 - \delta \quad\quad (\star) 
\qquad\quad\quad\quad \mu(\forall t : h_t \leq \hat h_t) \geq 1 - \delta\quad \quad (\star\star)
}
\end{theorem}

\begin{proof}
To prove ($\star$) 
define event 
$A \defined \set{\omega : \sup_{t} {\xi(\omega_{<t}) / \mu(\omega_{<t})} < {1 \over \delta}}$.
By Lemma \ref{lem:supermartingale} in the appendix we have that $\mu(A) \geq 1 - \delta$. If $\omega \in A$, then $\mu(\omega_{<t}) / \xi(\omega_{<t}) > \delta$ for all $t$ and
\eq{
c_t(\omega) 
&\sr{(a)}= \sum_{\nu \in \M} w_\nu {\nu(\omega_{<t}) \over \xi(\omega_{<t})} \KL{\omega_{<t}}{\nu}{\xi} 
\sr{(b)}\geq  w_\mu {\mu(\omega_{<t}) \over \xi(\omega_{<t})} \KL{\omega_{<t}}{\mu}{\xi} \\
&\sr{(c)}> w_\mu \cdot \delta \cdot \KL{\omega_{<t}}{\mu}{\xi} 
\sr{(d)}= w_\mu \cdot \delta \cdot d_t.
}
(a) is the definition of $c_t$. 
(b) follows by dropping all elements of the sum except $\mu$.
(c) by substituting the bound on $\mu / \xi$.
(d) is the definition of $d_t$. Therefore $d_t \cdot w_\mu \cdot \delta \leq c_t$ with $\mu$-probability at least $1 - \delta$ as required.
For ($\star\star$) we note that
by the definition of $\hat h_t$, if $\mu \in \M_t$, then
$h_t \leq \hat h_t$. The result is completed by applying Lemma \ref{lem:supermartingale} in the appendix to show 
that $\mu \in \M_t$ for all $t$ with probability at least $1 - \delta$. 
\end{proof}
\begin{theorem}\label{thm:KL-estimator-convergence}
The following hold:
\vspace{-0.2cm}
\begin{enumerate}
\item $\E_\mu \sum_{t=1}^\infty \hat d_t \leq {\ent(w) \over \delta w_\mu^2}$.
\item w.$\mu$.p.\ at least $1 - \delta$ it holds that
$\sum_{t=1}^\infty \hat d_t \leq {\ent(w) \over \delta^2 w_\mu^2}$.
\end{enumerate}
\end{theorem}

\begin{theorem}\label{thm:Hellinger-estimator-convergence}
The following hold:
\vspace{-0.2cm}
\begin{enumerate}
\item $\E_\mu \sum_{t=1}^\infty \hat h_t \leq {2 \over w_\mu} \left({\ln {1 \over w_\mu}} + \ln{1 \over \delta} + \ent(w)\right)$
\item w.$\mu$.p.\ at least $1 - \delta$,  
$\sum_{t=1}^\infty \hat h_t \leq {2 \over w_\mu} \left(2 \ln {1 \over w_\mu} + 5 \ln {1 \over \delta} + 3 \ent(w)\right)$.
\end{enumerate}
\end{theorem}
The consequences of Thereoms \ref{thm:curiosity-finite}, \ref{thm:KL-estimator-convergence} and \ref{thm:Hellinger-estimator-convergence} are summarised in 
Figure \ref{table:confidence} for both countable and finite hypothesis classes.
The proof of Theorem \ref{thm:KL-estimator-convergence} follows immediately from Theorem \ref{thm:curiosity} and Markov's inequality.
If $\M$ is finite and $w$ uniform, then one can use Theorem \ref{thm:curiosity-finite} instead to improve dependence on $\sfrac{1}{w_\mu}$.
For Theorem \ref{thm:Hellinger-estimator-convergence} we use Theorem \ref{thm:hellinger-concentration} and the following lemma, which
is a generalization of Lemma 4 in \cite{HM07}. 
\begin{lemma}\label{lem:hellinger-error}
Let $\kappa > 0$ and stopping time $\tau \defined \min_t \set{t : \nu_{<t} /\mu_{<t} < \kappa}$. Then
$\E_\mu \sum_{t=1}^{\tau-1} \hellinger{t}{\nu}{\mu} \leq 
2\ln\E_\mu \exp\left({1 \over 2} \sum_{t=1}^{\tau-1} \hellinger{t}{\nu}{\mu}\right) \leq
\ln{1 \over \kappa}$.
\end{lemma}

\iftecrep
\begin{proof}
We borrow some tricks from \cite{Vov87} and \cite[Lem 4]{HM07}.
Define $\rho$ inductively by
$\rho(a|x) \defined {\sqrt{\nu(a|x) \mu(a|x)}  / \sum_{b\in \A} \sqrt{\nu(b|x) \mu(b|x)}}$.
\eq{
\rho_{<\tau} &\sr{(a)}= \prod_{t=1}^{\tau-1} \rho_{t} 
\sr{(b)}\geq \prod_{t=1}^{\tau-1} \sqrt{\nu_{t} \mu_{t}} \exp\left({1 \over 2} \hellinger{t}{\nu}{\mu}\right) \\
&\sr{(c)}=\prod_{t=1}^{\tau-1} \mu_{t} \sqrt{{\nu_{t} \over \mu_{t}} } \exp\left({1 \over 2}{\hellinger{t}{\nu}{\mu}}\right)
\sr{(d)}=\mu_{<\tau}\sqrt{{\nu_{<\tau} \over \mu_{<\tau}} } \prod_{t=1}^{\tau-1} \exp\left({1 \over 2}{\hellinger{t}{\nu}{\mu}}\right) \\
&\sr{(e)}\geq \mu_{<\tau}\sqrt{\kappa} \prod_{t=1}^{\tau-1} \exp\left({1 \over 2}{\hellinger{t}{\nu}{\mu}} \right)
\sr{(f)}= \mu_{<\tau}\sqrt{\kappa} \exp\left({1 \over 2}{\textsum_{t=1}^{\tau-1} \hellinger{t}{\nu}{\mu}}\right)
}
where (a) and (d) follow from the definition of conditional probability. 
(b) by inserting the definition of $\rho_t$ and applying Lemma \ref{lem:hellinger}.
(c) by factoring.
(e) by noting that $\nu_{<\tau} / \mu_{<\tau} \geq \kappa$ by the definition of $\tau$.
(f) by exchanging $\prod \exp = \exp \sum$.
Therefore
$\sqrt{\kappa} \exp({1 \over 2} \textsum_{t=1}^{\tau-1} \hellinger{t}{\nu}{\mu})
\leq {\rho_{<\tau} \over \mu_{<\tau}}$.
Taking the expectation with respect to $\mu$
\eq{
\E_\mu \exp\left({1 \over 2} \textsum_{t=1}^{\tau-1} \hellinger{t}{\nu}{\mu}\right) 
&\leq \E_\mu {1 \over \sqrt{\kappa}} {\rho_{<\tau} \over \mu_{<\tau}} 
\leq \E_\rho {1 \over \sqrt{\kappa}} = {1 \over \sqrt{\kappa}}.
}
By Jensen's inequality $\E X = 2 \ln \exp \E{1 \over 2}X  \leq 2 \ln \E \exp{1 \over 2}X$ and so 
\eq{
\E_\mu \textsum_{t=1}^\tau \hellinger{t}{\nu}{\mu} 
\leq 2\ln \E_\mu \exp\left({1 \over 2} \textsum_{t=1}^\tau \hellinger{t}{\nu}{\mu}\right) 
\leq 2\ln {1 \over \sqrt{\kappa}} 
= \ln{1 \over \kappa}
}
as required.
\end{proof}
\fi

\begin{proofof}{ of Theorem \ref{thm:Hellinger-estimator-convergence}}
\iftecrep
Define stopping times $\tau_\nu$ and $\bar \tau_\nu$ by
\eq{
\tau_\nu &\defined \min_t \set{t : {\nu_{<t} / \xi_{<t}}<{w_\nu / w_\mu} \delta}
& \bar \tau_\nu &\defined \min_t \set{t : {\nu_{<t} / \mu_{<t}}<{w_\nu} \delta}
}
First we show that $\bar \tau_\nu \geq \tau_\nu$. By dominance $\xi_{<t} \geq w_\mu \mu_{<t}$ we have that
\eq{
{\nu_{<t} \over \mu_{<t}} < w_\nu \delta \implies {\nu_{<t} \xi_{<t} \over \mu_{<t} \xi_{<t}} < w_\nu\delta
\implies {\nu_{<t} \over \xi_{<t}} < {w_\nu \over w_\mu} \delta
}
Therefore $\bar \tau_\nu \geq \tau_\nu$. Let $\nu_t \defined \argmax_{\nu \in \M_t} \hellinger{t}{\nu}{\xi}$ and bound
\eqn{
\nonumber &\sum_{t=1}^\infty {\hat h_t} 
\sr{(a)}= \sum_{t=1}^\infty {w_{\nu_t} \over w_{\mu}} \hellinger{t}{\nu_t}{\xi} 
\sr{(b)}\leq {2 \over w_\mu} \sum_{t=1}^\infty \left({w_{\nu_t}} \hellinger{t}{\mu}{\xi} + w_{\nu_t} \hellinger{t}{\mu}{\nu_t} \right) \\
\label{eq:hellinger}
&\sr{(c)}\leq {2 \over w_\mu} \sum_{t=1}^\infty \left(\hellinger{t}{\mu}{\xi} +  \sum_{\mathclap{\nu \in \M_t}}w_\nu\hellinger{t}{\nu}{\mu}\right) \\ 
\nonumber
&\sr{(d)}= {2 \over w_\mu} \left(H_\infty + \sum_{\mathclap{\nu \in \M}} \sum_{t=1}^{\mathclap{\tau_\nu - 1}} w_\nu\hellinger{t}{\nu}{\mu}\right) 
\sr{(e)}\leq {2 \over w_\mu} \Bigg(H_\infty + \sum_{\nu \in \M} w_\nu \sum_{t=1}^{\bar\tau_\nu - 1} \hellinger{t}{\nu}{\mu}\Bigg) 
}
where (a) is the definition of $\hat h_t$. 
(b) follows by the inequality $\hellinger{t}{\nu_t}{\xi} < 2(\hellinger{t}{\nu_t}{\mu} + \hellinger{t}{\mu}{\xi})$.
(c) by dropping $w_{\nu_t} \leq 1$ and bounding the single term $w_{\nu_t} \hellinger{t}{\nu_t}{\mu}$ by the sum over all $\nu$ in the
plausible class $\M_t$.
(d) by the definitions of $H_\infty$, $\M_t$ and $\tau_\nu$.
(e) by the fact that $\bar\tau_\nu \geq \tau_\nu$ as previously shown.
\else
The proof is neccesarily brief with a complete version available in \cite{LHS13bayes-conc-tech}.
Define stopping time $\bar \tau_\nu \defined \min_t \set{t : \nu_{<t} / \mu_{<t} < w_\nu \delta}$, then it may be shown that 
\eqn{
\label{eq:hellinger}
\sum_{t=1}^\infty \hat h_t
\leq {2 \over w_\mu} \Bigg(H_\infty + \sum_{\nu \in \M} w_\nu \sum_{t=1}^{\bar\tau_\nu - 1} \hellinger{t}{\nu}{\mu}\Bigg) 
}
where we used the fact that ${1 \over 2} \hellinger{t}{\nu}{\xi} \leq  \hellinger{t}{\nu}{\mu} + \hellinger{t}{\mu}{\xi}$ and the definitions of $H_\infty$ and
$\bar \tau_\nu$.
\fi
Let $\Delta_\nu \defined \sum_{t=1}^{\bar \tau_\nu-1} \hellinger{t}{\nu}{\mu}$.
The first claim is proven by taking the expectation with respect to $\mu$ and substituting Theorem \ref{thm:solomonoff-stopping} to bound $\E_\mu H_\infty \leq \ln{\frac{1}{w_\mu}}$ 
and Lemma \ref{lem:hellinger-error} with $\tau = \bar \tau_\nu$ and $\kappa = w_\nu \delta$
to bound $\E_\mu \Delta_\nu \leq \ln \frac{1}{w_\nu} + \ln \frac{1}{\delta}$. For the high probability bound
let $\lambda_\nu \defined 3 \ln{1 \over \delta w_\nu} + \ln {1 \over w_\mu}$ and
apply Lemma \ref{lem:hellinger-error} and Markov's inequality.
\eq{
\P{\Delta_\nu \geq \lambda_\nu}
&= \P{e^{ \Delta_\nu / 2} \geq e^{\lambda_\nu/2}}
\leq e^{-\lambda_\nu/2} \E_\mu[ e^{\Delta_\nu/2} ]
\leq {e^{-\lambda_\nu/2} \over \sqrt{w_\nu \delta}} = w_\nu \delta.
} 
By Theorem \ref{thm:hellinger-concentration} we have that $H_\infty \leq \ln \sfrac{1}{w_\mu} + 2\ln \sfrac{1}{\delta w_\mu}$
with $\mu$-probability at least $1 - w_\mu \delta$ and
by the union bound and the fact that $\sum_{\nu} w_\nu = 1$ we obtain with probability at least $1 - \delta$ that
$\Delta_\nu \leq \lambda_\nu$ for all $\nu$ and $H_\infty \leq \ln\frac{1}{w_\mu} + 2 \ln{\frac{1}{\delta}}$, which
when substituted into \eqr{eq:hellinger} leads to
$\sum_{t=1}^\infty \hat h_t \leq \frac{2}{w_\mu} (2 \ln \frac{1}{w_\mu} + 5 \ln \frac{1}{\delta} + 3 \ent(w))$
as required.
\end{proofof}
\begin{figure}[H]
\extrarowsep=5pt
{
\scriptsize
\begin{tabu} to\linewidth{|X[1]X[4]X[4]|}
\hline
$|\M|$														& Expectation & High Probability \\ \hline
\multirow{2}{*}{$\infty$} 								& 
$\E_\mu \sum_{t=1}^\infty \hat d_t \lesssim {\ent(w) \over \delta w_\mu^2}$ &
$\sum_{t=1}^\infty \hat d_t \lesssim {\ent(w) \over \delta^2 w_\mu^2}$ \\
 & $\E_\mu \sum_{t=1}^\infty {\hat h_t} \lesssim {1 \over w_\mu}\left( \ent(w) + \ln{1 \over w_\mu \delta}\right) $ &
$\sum_{t=1}^\infty \hat h_t \lesssim {1 \over w_\mu} \left(\ent(w) + \ln{1 \over w_\mu \delta} \right)$ \\ \hline
\multirow{2}{*}{\begin{tabular}{l}\tiny $K$ \\ $w_\nu\! =\! {1\over K}$ \end{tabular}} 								& 
$\E_\mu \sum_{t=1}^\infty \hat d_t \lesssim {{K \over \delta} \ln^2 K}$ &
$\sum_{t=1}^\infty \hat d_t \lesssim {{K \over \delta^2} \ln^2 K}$ \\
  & $\E_\mu \sum_{t=1}^\infty {\hat h_t} \lesssim K \left(\ln K + \ln{1 \over \delta}\right) $ &
$\sum_{t=1}^\infty \hat h_t \lesssim K \left(\ln K + \ln{1 \over \delta}\right)$ \\ \hline
\end{tabu} \\[1pt]
\hfill {$\lesssim$ ignores constant multiplicative factors}
}
\caption{Confidence bounds}
\label{table:confidence}
\end{figure}

%%%%%%%%%%%%%%%%%%%%%%%%%%%%%%%%%%%%%%%%%%%%%%%%%%%%%%%%%
%% KWIK Learning
%%%%%%%%%%%%%%%%%%%%%%%%%%%%%%%%%%%%%%%%%%%%%%%%%%%%%%%%%
\section{KWIK Learning}\label{sec:kwik}

\begin{wrapfigure}[8]{r}{6cm}
\small
\begin{minipage}{6cm}
\begin{algorithm}[H]
\caption{KWIK Learner}
\scriptsize
\begin{algorithmic}[1]
\State {\bf Inputs:} $\varepsilon$, $\delta$ and $\M \defined \left\{\nu_1, \nu_2,\cdots,\nu_{K}\right\}$.
\State $t \leftarrow 1$ and $\omega_{<t} \leftarrow \epsilon$ and $w_\nu = {1\over K}$
\Loop
\If{$\hat h_t(\omega_{<t}) \leq \varepsilon$}
\State output $\xi(\cdot|\omega_{<t})$
\Else
\State output $\bot$
\EndIf
\State observe $\omega_t$ and $t \leftarrow t + 1$ 
\EndLoop
\end{algorithmic}
\end{algorithm}
\end{minipage}
\end{wrapfigure}
The KWIK learning framework involves an environment and agent interacting sequentially as depicted below.
Suppose $|\M| = K < \infty$ and $\varepsilon,\delta > 0$ are known to both parties. 
A run starts with the environment choosing an unknown $\mu \in \M$. At each time-step $t$ thereafter
the agent chooses between outputting a predictive distribution $\rho(\cdot|\omega_{<t})$ and special symbol $\bot$.
The run is failed if the agent outputs $\rho$ and $\hellinger{\omega_{<t}}{\rho}{\mu} > \varepsilon$,
otherwise $\omega_t$ is observed and the run continues. 
An agent is said to be KWIK if it fails the run with probability at most $\delta$ and chooses $\bot$ at most
$B(\varepsilon, \delta)$ times with probability at least $1 - \delta$.
Ideally, $B(\varepsilon,\delta)$ should be polynomial in $\sfrac{1}{\varepsilon}$ and $\sfrac{1}{\delta}$ \cite{LLWS11}.

\begin{figure}[H]
\scriptsize
\centering
\begin{tikzpicture}[draw,>=stealth',thick,scale=0.8]
\tikzstyle{box}=[draw,minimum height=0.5cm,minimum width=2.3cm]
\tikzstyle{line}=[draw,->]

\node[box] (confident) at (0,2.1) {Am I confident?};
\node[box] (output) at(-4,2.1) {output $\rho(\cdot|\omega_{<t})$};
\node[box] (clueless) at(4,2.1) {output $\bot$};
\node[box] (correct) at (-4,3.8) {$\hellinger{\omega_{<t}}{\rho}{\mu} \leq \varepsilon$?};
\node[box] (give) at (0,3.8) {present $\omega_t$ to agent};
\node[box] (fail) at (-4,5.0) {agent fails run};
\node[box] (choose) at (6,3.8) {choose $\mu \in \M$};
\node[rotate=90] (agent) at (-6,1.9) {Agent};
\node[rotate=90] (env) at (-6,4.3) {Environment};

\path[line] (confident) -- node[below] {yes} (output);
\path[line] (confident) -- node[below] {no} (clueless);
\path[line] (output) -- (correct);
\path[line] (clueless) |- (give);
\path[line] (correct) -- node[above] {yes} (give);
\path[line] (correct) -- node[left] {no} (fail);
\path[line] (give) -- (confident);
\path[draw] (choose) |- ($(confident.south)!-0.3cm!(confident.north)$);
\path[line] ($(confident.south)!-0.3cm!(confident.north)$) -- (confident);
\draw[line width=0.3cm,white] (-5.9,2.9) -- (7.9,2.9);
\draw[dashed] (-6,2.9) -- (7.9,2.9);
\end{tikzpicture}
\caption{KWIK learning framework}
\label{fig:kwik}
\end{figure}
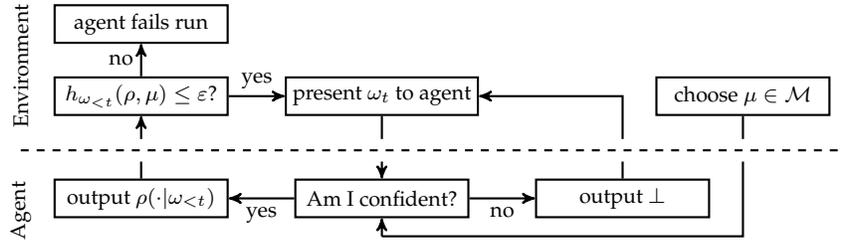

\begin{theorem}
Algorithm 1 is KWIK.
\end{theorem}
\begin{proof}
By Theorem \ref{thm:confidence}, Algorithm 1 fails a run with probability at most $\delta$.
Using $\gtrsim$ to ignore constant multiplicative factors, by Theorem \ref{thm:Hellinger-estimator-convergence} we have that
\eq{
\P{\left|\set{t : \hat h_t \geq \varepsilon}\right| \gtrsim {K \over \varepsilon} \ln{K \over \delta}}
&\leq \P{\textsum_{t=1}^\infty \hat h_t \gtrsim {K}\ln {K \over \delta}}
\leq \delta.
}
Therefore the agent will choose $\bot$ at most
$\O{{K \over \varepsilon} \ln {K \over \delta}}$ times
with probability at least  $1 - \delta$.  
\end{proof}
The Hellinger distance upper bounds the total variation distance. 
$\totalvar{x}{\mu}{\xi} = {1 \over 2}\sum_{a \in \A} |\mu(a|x) - \xi(a|x)| \leq \sqrt{\hellinger{x}{\mu}{\xi}}$.
Therefore if Algorithm 1 is run with $\varepsilon = \varepsilon_1^2$, then with high probability when predicting it will be 
$\varepsilon_1$-optimal with respect to the total variation
distance and it will output $\bot$ at most $\O{{K \over \varepsilon^2_1} \ln {K \over \delta}}$ times, which is the same bound achieved by the $k$-meteorologist algorithm \cite{DLL09}.

%%%%%%%%%%%%%%%%%%%%%%%%%%%%%%%%%%%%%%%%%%%%%%%%%%%%%%%%%
% CONCLUSIONS
%%%%%%%%%%%%%%%%%%%%%%%%%%%%%%%%%%%%%%%%%%%%%%%%%%%%%%%%%
\section{Conclusions}

The bound on the squared Hellinger distance $\hat h_t$ is especially nice because the results are rather clean.
While the super-linear dependence on the size of the model class in Figure \ref{table:confidence} is unfortunate,
it is a worst-case bound that is only achieved when at each time-step only one model differs from $\xi$ 
(see the proof of Proposition \ref{prop:curiosity-lower} for an example environment class when this occurs).
For Bernoulli classes the estimator performs comparably with the Hoeffding bound 
\iftecrep(Appendix \ref{app:experiments}). \else \cite{LHS13bayes-conc-tech}. \fi
In the case when $\M$ is countable $\hat h_t$ is independent of $\mu$, but not $w_\mu$, which is also typically unknown.
Either choose a conservatively small $w$ and pay the ${1 \over w}\ln \frac{1}{w}$ price, or decrease $w$ with $t$ at some slow rate, say $w = \sqrt{\text{\sfrac{1}{$t$}}}$. 
Analyzing this situation is interesting future work.

There is opportunity for some improvement on the bound $\hat d$. Intuitively we expect the real dependence on $\sfrac{1}{\delta}$ 
ought to be logarithmic, not linear. The unimprovable result of Theorem \ref{thm:probability-bound} is interesting when compared to Theorem \ref{thm:hellinger-concentration}. Researchers frequently bound
the total variation distance via the KL divergence. These results show that this is sometimes weaker than using the Hellinger distance when high-probability bounds are required.

KWIK learning for sequence prediction 
was chosen because our new results can easily be applied to prove a state-of-the-art bound in that setting. Although we have the same theoretical guarantee as the $k$-meteorologist algorithm \cite{DLL09}, 
our simple algorithm 
eliminates environments smoothly as they become unlikely while in that work no model (expert) is discarded before at least $m = O(\frac{1}{\varepsilon^2} \ln{\frac{1}{\delta}})$ differentiating
samples have been observed. This distinction makes us suspect that Algorithm 1 may perform more efficiently in practice. Additionally, assuming $\nu(\cdot|x)$ can be computed in constant time, then
Algorithm 1 runs in $O(K)$ time per time-step, while a naive implementation of the $k$-meteorologist algorithm 
appears to have $O(K^2)$ running time per time-step. 

Finally, we want to emphasize the generality of the results, especially 
Theorem \ref{thm:Hellinger-estimator-convergence}, which although tight in a minimax sense, can likely be improved in easier cases without changing the definition of $\hat h_t$. 
An interesting continuation is the parametric case that is intuitively straight-forward, but technically challenging (see \cite{BC90} and \cite[\S3]{Hut05} for some of the required techniques). 

%%%%%%%%%%%%%%%%%%%%%%%%%%%%%%%%%%%%%%%%%%%%%%%%%%%%%%%%%
%% BIBLIOGRAPHY
%%%%%%%%%%%%%%%%%%%%%%%%%%%%%%%%%%%%%%%%%%%%%%%%%%%%%%%%%
\bibliography{all}

%%%%%%%%%%%%%%%%%%%%%%%%%%%%%%%%%%%%%%%%%%%%%%%%%%%%%%%%%
%% APPENDIX
%%%%%%%%%%%%%%%%%%%%%%%%%%%%%%%%%%%%%%%%%%%%%%%%%%%%%%%%%

\appendix

%%%%%%%%%%%%%%%%%%%%%%%%%%%%%%%%%%%%%%%%%%%%%%%%%%%%%%%%%
%% TECHNICAL LEMMAS
%%%%%%%%%%%%%%%%%%%%%%%%%%%%%%%%%%%%%%%%%%%%%%%%%%%%%%%%%
\iftecrep
\section{Technical Lemmas}

\begin{lemma}[\citen{Vov87}]\label{lem:hellinger}
Let $p$ and $q$ be distribution on $\A$, then
\eq{
\textsum_{a \in \A} \sqrt{p(a) q(a)} \leq \exp\left(-{1 \over 2}\textsum_{a \in \A} \left(\sqrt{p(a)} - \sqrt{q(a)}\right)^2\right). % [TODO] Check this and application
}
\end{lemma}

\begin{lemma}[\citen{Vil39}]\label{lem:supermartingale}
If $z_{<t} \defined \xi_{<t} / \mu_{<t}$, then
$z_{<t}$ is a $\mu$-super-martingale, $\P{\lim_{t\to\infty} z_{<t} < \infty} = 1$ and
$\P{\sup_{t} z_{<t} \geq {1 \over \delta}} \leq \delta$.
\end{lemma}

\begin{proof}
The proof is straight-forward and is included for the sake of completeness.
Define $A \subset X^*$ to be the set of finite strings defined by
\eq{
A \defined \set{x \in X^* : \xi(x) / \mu(x) \geq 1/ \delta \wedge \forall t \leq \ell(x), \xi(x_{<t}) / \mu(x_{<t}) < 1/\delta }
}
So $A$ is the set of finite strings where $\xi(x) / \mu(x)$ first drops below $\delta$. Let $\omega \in X^\infty$ and $z_t(\omega) \geq 1/\delta$, then
there exists a $t$ such that $\omega_{1:t} \in A$. Therefore if $\bar A = \set{x \in A : \mu(x) > 0}$, then 
\eq{
\mu(\lim_{t\to\infty} z_t \geq 1/\delta) 
\sr{(a)}= \sum_{x \in A} \mu(x)
\sr{(b)}= \sum_{x \in \bar A} \mu(x)
\sr{(c)}= \sum_{x \in \bar A} \mu(x) \xi(x) / \xi(x)
\sr{(d)}\leq \delta \sum_{x \in \bar A} \xi(x)
\sr{(e)}\leq \delta
}
where (a) follows from the definition of $A$ and $z_t$. 
(b) since $\mu(x) = 0$ when $x \in A - \bar A$.
(c) since $\xi(x) \geq w_\mu \mu(x) > 0$ for $x \in \bar A$.
(d) by the bound $\mu(x) / \xi(x) \leq \delta$ for $x \in A$.
(e) since $\xi$ is a measure.
\end{proof}

\begin{lemma}\label{lem:exists}
Both $\E_\mu \ln (\mu_{<n} / \xi_{<n})$ and $\E_\mu d_n$ exist and are finite.
\end{lemma}

\iffalse
\begin{proofsketchof}{}
Bound $\E_\mu |\ln {\mu_{<n} / \xi_{<n}}| = \sum_{x \in \A^n} \mu(x) |\ln{\mu(x) / \xi(x)}|$ by
dividing $\A^n$ into $A$ where $\mu(x) / \xi(x) \geq 1$ and $B$ where $\mu(x) / \xi(x) < 1$, being careful to deal with the case when
$\mu(x) = 0$.
Then write out the definition of expectation and bound $\mu(x) / \xi(x)$ by $1/w_\mu$ if $x \in A$ and $\ln \xi(x) / \mu(x) \leq \xi(x) / \mu(x)$ for $x \in B$. The second quantity is bounded similarly.
\end{proofsketchof}
\fi

\begin{proof}
Let $A \defined \set{x \in \A^{n-1} : \mu(x) > 0}$ and $B = \set{x \in A : \mu(x) / \xi(x) \geq 1}$ and $C = \set{x \in A :\mu(x) / \xi(x) < 1}$.
If $\one_B$ is the indicator event $\one_B(\omega) \defined \ind{\omega \in \Gamma_x : x \in B}$ and $\one_C$ is defined similarly, then
\eq{
\E_\mu \left[|\ln \mu_{<n} / \xi_{<n}| \right]
&= \E_\mu\left[ \one_{B}\ln \mu_{<n} / \xi_{<n}\right] + \E_\mu \left[\one_C \ln \xi_{<n} / \mu_{<n}\right] \\
&\leq \E_\mu\left[-\one_B \ln{w_\mu}\right] + \E_\mu \left[\one_C \xi_{<n} / \mu_{<n}\right] 
\leq -\ln w_\mu + 1
}
where the first inequality is due to the dominance $\xi_{<n} \geq w_\mu \mu_{<n}$ and the inequality $\ln x \leq x$ for all $x \geq 1$. The second
inequality follows from basic properties of expectation.
For the second part note that $d_{n+1}$ is positive and by \eqr{eq:KL} that $d_{n} \leq \ln \xi_{<n} / \mu_{<n} - \ln w_\mu  \leq |\ln \mu_{<n} / \xi_{<n}| - \ln w_\mu$. Then proceed as in the first part.
\end{proof}
\fi

%%%%%%%%%%%%%%%%%%%%%%%%%%%%%%%%%%%%%%%%%%%%%%%%%%%%%%%%%
%% Proof of Solomonoff Bound
%%%%%%%%%%%%%%%%%%%%%%%%%%%%%%%%%%%%%%%%%%%%%%%%%%%%%%%%%
\iftecrep
\section{Proof of Theorem \ref{thm:solomonoff-stopping}}\label{app:solomonoff}
First we note that the squared Hellinger distances is bounded by the KL divergence, so $H_\infty \leq D_\infty$.
We now bound $\E_\mu D_\infty$,
which follows from the chain rule for the conditional relative entropy. Fix $n \in \N$ and assume that
\eqn{
\tag{$\star$}
\Delta_{n-1} \defined \E_\mu \textsum_{t=1}^n d_t = \E_\mu \ln{\mu_{<n} \over \xi_{<n}},
}
which is easily verified when $n = 1$. Therefore
\eq{
\Delta_{n} 
&\sr{(a)}= \E_{\mu} \ln{\mu_{<n} \over \xi_{<n}} + \E_\mu d_{n} 
\sr{(b)}=\E_\mu\left[\E_\mu\left[\ln{\mu_{1:n} \over \xi_{1:n}}\bigg|\F_{<n}\right]\right] 
\sr{(c)}= \E_\mu \ln{\mu_{1:n} \over \xi_{1:n}}.
}
(a) holds by Lemma \ref{lem:exists}.
(b) by \eqr{eq:KL} and the definition of expectation.
(c) by the definition of (conditional) expectation.
Therefore $(\star)$ holds for all $n$ by induction.
By substituting dominance $\xi_n \geq w_\mu \mu_n$ into $(\star)$ one obtains that $\Delta_n \leq -\E_\mu \ln w_\mu = -\ln w_\mu$.
The proof is completed by taking the limit as $n \to\infty$ and applying the Lebesgue monotone convergence theorem to show
that $\E_\mu D_\infty = \lim_{n\to\infty} \Delta_n \leq -\ln w_\mu$. That $d_t$ and $h_t$ converge to $0$ with $\mu$-probability $1$
follows from Markov's inequality applied to $D_\infty$ and $H_\infty$ respectively.
\fi

%%%%%%%%%%%%%%%%%%%%%%%%%%%%%%%%%%%%%%%%%%%%%%%%%%%%%%%%%
%% Proof of Curiosity Bound
%%%%%%%%%%%%%%%%%%%%%%%%%%%%%%%%%%%%%%%%%%%%%%%%%%%%%%%%%
\section{Proof of Theorem \ref{thm:curiosity-finite}}
\iftecrep
\else
\begin{lemma}[\citen{Vil39}]\label{lem:supermartingale}
If $z_{<t} \defined \xi_{<t} / \mu_{<t}$, then
$z_{<t}$ is a $\mu$-super-martingale, $\P{\lim_{t\to\infty} z_{<t} < \infty} = 1$ and
$\P{\sup_{t} z_{<t} \geq {1 \over \delta}} \leq \delta$.
\end{lemma}
\fi

If $t \leq t'$ are stopping times, then
$I(\omega) = [t(\omega), t'(\omega))$ is called a stopping interval
and $\A(I) \defined \A(t)$ is the set of finite sequences when the start of $I$ becomes known. 
If $\rho$ is a measure, then $\rho(I) \defined \sum_{x \in \A(I)} \rho(x)$ is the $\rho$-probability of encountering interval $I$ at some point.

\begin{lemma}\label{lem:solomonoff-interval}
Let $\nu \in \M$ and $I$ be a stopping interval. Then
\eq{
\E_\nu \textsum_{t \in I} \KL{t}{\nu}{\xi} \leq \textsum_{x \in \A(I)} \nu(x) \left(\ln {1 \over w_\nu} + \ln {\xi(x) \over \nu(x)}\right). 
}
\end{lemma}

\begin{proofof}{}
The result follows from Theorem \ref{thm:solomonoff-stopping} and definitions. Let $t$ be the stopping time
governing the start of interval $I$. Then
\eq{
\E_\nu \textsum_{t \in I} \KL{t}{\nu}{\xi} 
&\sr{(a)}= \textsum_{\mathclap{x \in \A(I)}} \nu(x)\E_\nu \left[ \textsum_{t \in I} \KL{t}{\nu}{\xi}\Bigg\vert x \right] \\ 
&\sr{(b)}\leq \textsum_{\mathclap{x \in \A(I)}} \nu(x)\E_\nu \left[ \textsum_{t=\ell(x)+1}^\infty \KL{t}{\nu}{\xi}\Bigg\vert x \right] \\ 
&\sr{(c)}\leq \textsum_{\mathclap{x \in \A(I)}} \nu(x) \ln {1 \over w_\nu(x)} 
\sr{(d)}= \textsum_{\mathclap{x \in \A(I)}} \nu(x) \left(\ln {1 \over w_\nu} + \ln {\xi(x) \over \nu(x)}\right).
}
(a) follows by by the definition of expectation.
(b) by increasing the size of the interval.
(c) follows from Theorem \ref{thm:solomonoff-stopping} by noting that $\xi(\cdot|x)$ is a mixture over
$\set{\nu(\cdot|x) : \nu \in \M}$ with prior $w(\cdot|x)$.
(d) because $w_\nu(x) = w_\nu \nu(x) / \xi(x)$ and by expanding the logarithm.
\end{proofof}

\begin{proofof}{ of Theorem \ref{thm:curiosity-finite}}
First, the quantity to be bounded can be rewritten as an average of $\nu$-expectations of a certain random variable. 
\eq{
\Delta &\defined \E_\mu \textsum_{t=1}^\infty c_t 
\sr{(a)}= \textsum_{t=1}^\infty \E_\mu c_t 
\sr{(b)}= \textsum_{t=1}^\infty \textsum_{x \in \A^{t-1}} \mu(x) \textsum_{\nu \in \M} {1 \over K} \cdot {\nu(x) \over \xi(x)} \KL{x}{\nu}{\xi} \\
&\sr{(c)}= {1 \over K}\textsum_{\nu \in \M} \textsum_{t=1}^\infty \textsum_{x \in \A^{t-1}} \nu(x) {\mu(x) \over \xi(x)} \KL{x}{\nu}{\xi} 
\sr{(d)}= {1 \over K}\textsum_{\nu \in \M} \textsum_{t=1}^\infty \E_\nu {\mu_{<t} \over \xi_{<t}} \KL{t}{\nu}{\xi} \\
&\sr{(e)}= {1 \over K}\textsum_{\nu \in \M} \underbrace{\E_\nu \textsum_{t=1}^\infty {\mu_{<t} \over \xi_{<t}} \KL{t}{\nu}{\xi}}_{\Delta(\nu)}.
}
(a) follows by the linearity of expectation and positivity of $c_t$.
(b) by writing out the definition of the expectation.
(c), (d) and (e) exchanging sums and the definition of expectation.
Define $a_t,b_t:\A^\infty \to \N$ by
\eq{
a_t(\omega) &\defined \sup_{t' \leq t} \floor{\ln \xi(\omega_{<t'}) / \nu(\omega_{<t'})} \quad &
b_t(\omega) &\defined \sup_{t' \leq t} \floor{\ln \mu(\omega_{<t'}) / \xi(\omega_{<t'})},
}
which are monotone non-decreasing. By the definition of $\xi$ as a uniform mixture over $\M$, $\mu(x) / \xi(x) \leq K$, so
$b_t(\omega) \leq \ln K \rdefined L$. Furthermore, $\mu(\epsilon) = \nu(\epsilon) = \xi(\epsilon) = 1$ implies
that $a_t(\omega), b_t(\omega) \geq 0$.
Define intervals of the following form
\eq{
I_{\beta}(\omega) &\defined \set{t : b_t = \beta \wedge a_t \leq \beta} \quad & 
I_{\alpha,\beta}(\omega) &\defined \set{t : a_t = \alpha \wedge b_t = \beta}.
}
Then $\N$ can be divided into disjoint intervals of the form $I_{\beta}$ and $I_{\alpha,\beta}$ where $\alpha > \beta$. 
\eqn{
\label{C3:eq:disjoint} \forall(\omega \in \A^\infty),\quad  \N = \bigcup_{\beta=0}^L \left(I_{\beta}(\omega) \cup \bigcup_{\alpha > \beta \in \N} I_{\alpha,\beta}(\omega)\right)
}
\iftecrep
See Figure \ref{C3:fig:talphabeta} for an example of the definition of $I_{\beta}$ and $I_{\alpha,\beta}$.
\fi
Then $\Delta(\nu)$ can be decomposed as follows
\eq{
\Delta(&\nu) \equiv \E_\nu \textsum_{t=1}^\infty {\mu_{<t} \over \xi_{<t}} \KL{t}{\nu}{\xi} \\
&=\underbrace{\textsum_{\beta=0}^L \E_\nu \textsum_{t \in I_\beta} {\mu_{<t} \over \xi_{<t}} 
\KL{t}{\nu}{\xi}}_{\Delta_{1}(\nu)} 
 + \underbrace{\textsum_{\beta=0}^L \textsum_{\alpha=\beta+1}^\infty \E_\nu \textsum_{t \in I_{\alpha,\beta}}   
{\mu_{<t} \over \xi_{<t}} 
\KL{t}{\nu}{\xi}}_{\Delta_2(\nu)}
}
where the second equality follows from \eqr{C3:eq:disjoint} and by linearity of the expectation.
We now bound $\Delta_1(\nu)$ and $\Delta_2(\nu)$.
\eq{
\Delta_{1}(\nu) &\equiv
\textsum_{\beta=0}^L \E_\nu \textsum_{t\in I_{\beta}}{\mu_{<t} \over \xi_{<t}} \KL{t}{\nu}{\xi} 
\sr{(a)}\leq \textsum_{\beta=0}^L e^{\beta+1} \E_\nu \textsum_{t \in I_\beta} \KL{t}{\nu}{\xi}\\
&\sr{(b)}\leq \textsum_{\beta=0}^L e^{\beta+1} \textsum_{\mathclap{x \in \A(I_{\beta})}} \nu(x)\left(L + \ln {\xi(x) \over \nu(x)}\right) 
\sr{(c)}\leq \textsum_{\beta=0}^L e^{\beta+1} \nu(I_{\beta}) \left(L + \beta + 1\right).
}
(a) follows since on the interval $I_{\beta}$ the quantity $\mu_{<t} / \xi_{<t} < e^{\beta+1}$.
(b) follows from Lemma \ref{lem:solomonoff-interval} and by noting that $\ln \nicefrac{1}{w_\mu} = \ln K = L$.
(c) by the definition of $\nu(I_{\beta})$ and because $\xi_{<t} / \nu_{<t} < e^{\beta+1}$ on the interval $I_{\beta}$.
$\Delta_2(\nu)$ is bounded in a similar fashion.
\eq{
\Delta_{2}(\nu) &\equiv
\textsum_{\beta=0}^L \textsum_{\alpha=\beta+1}^\infty \E_\nu \textsum_{{t \in I_{\alpha,\beta}}}
{\mu_{<t} \over \xi_{<t}} \KL{t}{\nu}{\xi} 
\sr{(a)}\leq
\textsum_{\beta=0}^L e^{\beta+1} \textsum_{\mathclap{\alpha=\beta+1}}^\infty \;\E_\nu \textsum_{\mathclap{t \in I_{\alpha,\beta}}} 
\KL{t}{\nu}{\xi} \\
&\sr{(b)}\leq 
\textsum_{\beta=0}^L e^{\beta+1} \textsum_{\alpha=\beta+1}^\infty \textsum_{x \in \A(I_{\alpha,\beta})} \nu(x) \left(L + \ln {\xi(x) \over \nu(x)}\right) \\
&\sr{(c)}\leq 
\textsum_{\beta=0}^L e^{\beta+1} \textsum_{\mathclap{\alpha=\beta+1}}^\infty \nu(I_{\alpha,\beta}) \left(L + \alpha + 1 \right) 
\sr{(d)}\leq
\textsum_{\beta=0}^L e^{\beta+1} \textsum_{\mathclap{\alpha=\beta+1}}^\infty e^{-\alpha} \left(L + \alpha + 1 \right) \\
&\sr{(e)}= \textsum_{\beta=0}^L e^{\beta+1} e^{-\beta}\left(L + \beta + 3\right) 
\sr{(f)}= 3(L+1)(L + 2).
}
(a) follows because $\mu_{<t} / \xi_{<t} < e^{\beta+1}$ on the interval $I_{\alpha,\beta}$ and by expanding the interval.
(b) by Lemma \ref{lem:solomonoff-interval}.
(c) because $\nu(I_{\alpha,\beta}) = \sum_{x \in \A(I_{\alpha,\beta})} \nu(x)$.
By definition, if $x \in \A(I_{\alpha,\beta})$, then a $\xi(x) / \nu(x) \geq e^{\alpha}$. By Lemma \ref{lem:supermartingale} the $\nu$-probability of
this ever occurring is at most $e^{-\alpha}$, which implies $\nu(I_{\alpha,\beta}) \leq e^{-\alpha}$ and so gives (d).
(e) and (f) follow from simple algebra.
Combining the bounds of $\Delta_1(\nu)$ and $\Delta_2(\nu)$ leads to
\eq{
\textsum_{\nu \in \M} &w_\nu \Delta(\nu) 
\equiv \textsum_{\nu \in \M} w_\nu \left(\Delta_1(\nu) + \Delta_2(\nu)\right) \\
&\sr{(a)}\leq 3(L+1)(L + 2) + \textsum_{\nu \in \M} w_\nu \textsum_{\beta=0}^L e^{\beta+1} \nu(I_{\beta}) (L + \beta + 1) \\
&\sr{(b)}= 3(L + 1)(L+2) + \textsum_{\beta=0}^L e^{\beta+1} \xi(I_{\beta})(L + \beta + 1) \\
&\sr{(c)}\leq 3(L + 1)(L + 2) + \textsum_{\beta=0}^L 2(L + \beta + 1) 
\sr{(d)}=6L^2 + 14L + 8
}
(a) by substituting the bounds for $\Delta_1(\nu)$ and $\Delta_2(\nu)$.
(b) by exchanging sums and recalling that $\sum_{\nu \in \M} w_\nu \nu(A) = \xi(A)$ for all measurable $A$.
(c) from Lemma \ref{lem:supermartingale} applied to bound $\xi(I_{\beta}) \leq e^{-\beta}$ in the same way as $\nu(I_{\alpha,\beta})$ was bounded.
(d) by simple algebra. The theorem is completed by substituting $L \defined \ln K$.
\end{proofof}

\iftecrep
\begin{figure}[H]
\centering
\begin{tabu}{|l|p{0.6cm}|p{0.6cm}|p{0.6cm}|p{0.6cm}|p{0.6cm}|p{0.6cm}|p{0.6cm}|p{0.6cm}|}\hline
time, $t$           & 1      &    2 &     3  &     4 &     5   &     6 &     7 &     8 \\ \hline
$\xi(\omega_{<t})/\nu(\omega_{<t})$ & 1      &    2 &     5  &     5 &     5   &     3 &     2 &     3 \\ 
$\mu(\omega_{<t})/\xi(\omega_{<t})$ & 1      &    2 &     2  &     1 &     5   &     2 &     1 &     9 \\ \hline
$a_t$									& 0				&		1 &			2	 &		 2 &     2   &     2 &     2 &     2 \\
$b_t$									& 0      &    1 &     1  &     1 &     2   &     2 &     2 &     3 \\ \hline 
											& 
											\multicolumn{1}{c|}{$I_0$}  &  
											\multicolumn{1}{c|}{$I_1$} &   
											\multicolumn{2}{c|}{$I_{2,1}$} &    
											\multicolumn{3}{c|}{$I_{2}$} &
											\multicolumn{1}{c|}{$I_3$}
											\\
\hline
\end{tabu}
\caption{Example $I_{\alpha,\beta}$ and $I_{\beta}$}
\label{C3:fig:talphabeta}
\end{figure}
\fi

%%%%%%%%%%%%%%%%%%%%%%%%%%%%%%%%%%%%%%%%%%%%%%%%%%%%%%%%%
%% Proof of Lower Bound on Curiosity
%%%%%%%%%%%%%%%%%%%%%%%%%%%%%%%%%%%%%%%%%%%%%%%%%%%%%%%%%
\section{Proof of Proposition \ref{prop:curiosity-lower}}

Let $\A = \set{0,1}$ and define measure $\nu^k$ to be the deterministic measure producing $k$ ones followed by zeros
$\nu^k(1|x) \defined \ind{\ell(x) < k}$.
Let $\M \defined \set{\nu^k : 0 \leq k \leq K - 1}$ and the true measure be $\mu \defined \nu_{K-1}$.
The Bayes mixture over $\M$ under the uniform prior becomes
$\xi(x) \defined {1 \over K} \sum_{k=0}^{K-1} \nu^k(x)$.
If $t < K$, then by substituting definitions one obtains
$\xi(1^t) = {(K - t) / K}$ and
$\xi(0|1^t) = {1 /(K - t)}$.
Therefore
\eq{
\E_\mu &\textsum_{t=1}^\infty c_t 
\sr{(a)}\geq \E_\mu \textsum_{t=1}^K c_t 
\sr{(b)}=\textsum_{t=0}^{K-1} \textsum_{k=0}^{K-1} {1 \over K} {\nu^k(1^t) \over \xi(1^t)} \dstyle \KL{1^t}{\nu^k}{\xi}  \\
&\sr{(c)}\geq\textsum_{t=0}^{K-1} {\nu^t(1^t) \over K \xi(1^t)} \dstyle \KL{1^t}{\nu^t}{\xi} 
\sr{(d)}=\textsum_{t=1}^K  {\ln{t} \over t} 
\sr{(e)}\geq {1 \over 2} {\ln K} - 1. 
}
(a) follows by truncating the sum and positivity of $c_t$.
(b) by the definition of $c_t$, the expectation and because $\mu(1^t) = 1$ for all $t \leq K-1$.
(c) by dropping all terms in the sum over $k$ except for $k = t$ and positivity of all quantities.
(d) and (e) follow by substituting definitions and simple calculus/algebra.

%%%%%%%%%%%%%%%%%%%%%%%%%%%%%%%%%%%%%%%%%%%%%%%%%%%%%%%%%%%%%%%
%  IID
%%%%%%%%%%%%%%%%%%%%%%%%%%%%%%%%%%%%%%%%%%%%%%%%%%%%%%%%%%%%%%%

\iftecrep
\iffalse
\section{Identical and Independently Distributed Data}
\fi
\fi

%%%%%%%%%%%%%%%%%%%%%%%%%%%%%%%%%%%%%%%%%%%%%%%%%%%%%%%%%%%%%%%
%  EXPERIMENTS
%%%%%%%%%%%%%%%%%%%%%%%%%%%%%%%%%%%%%%%%%%%%%%%%%%%%%%%%%%%%%%%

\iftecrep

%%%%%%%%%%%%%%%%%%%%%%%%%%%%%%%%%
% DATA
%%%%%%%%%%%%%%%%%%%%%%%%%%%%%%%%%
\pgfplotstableread{
0	0.187436	0.765429
1	0.282247	0.730414
2	0.344673	0.574539
3	0.389856	0.527864
4	0.219038	0.463775
5	0.258548	0.443561
6	0.290370	0.431398
7	0.205356	0.388267
8	0.235324	0.361352
9	0.260965	0.344351
10	0.197107	0.339426
11	0.221046	0.334896
12	0.169717	0.319714
13	0.191907	0.296634
14	0.211803	0.301471
15	0.169713	0.289406
16	0.188374	0.273105
17	0.205393	0.274737
18	0.169714	0.261036
19	0.185815	0.262156
20	0.154473	0.253916
21	0.169714	0.250403
22	0.141778	0.245454
23	0.156224	0.238527
24	0.169714	0.235544
25	0.144749	0.232391
26	0.157614	0.228989
27	0.134864	0.223054
28	0.147149	0.220260
29	0.126257	0.216535
30	0.138004	0.210858
31	0.149128	0.207117
32	0.129941	0.205313
33	0.140624	0.202072
34	0.122777	0.198266
35	0.133049	0.198489
36	0.142844	0.189967
37	0.126256	0.192411
38	0.135706	0.187505
39	0.120130	0.186799
40	0.129255	0.183424
41	0.114575	0.181324
42	0.123395	0.178271
43	0.109516	0.176941
44	0.118048	0.175354
45	0.126256	0.174874
46	0.113149	0.173507
47	0.121109	0.171000
48	0.108644	0.169037
49	0.116368	0.167101
50	0.104486	0.165048
51	0.111988	0.164535
52	0.100637	0.163231
53	0.107928	0.159837
54	0.114984	0.159811
55	0.104154	0.157982
56	0.111024	0.157361
57	0.100636	0.155767
58	0.107329	0.154228
59	0.097350	0.152656
60	0.103874	0.150637
61	0.094273	0.150375
62	0.100636	0.148339
63	0.106822	0.147021
64	0.097595	0.146438
65	0.103636	0.145088
66	0.094734	0.144872
67	0.100636	0.143887
68	0.092036	0.142413
69	0.097806	0.142540
70	0.089489	0.141095
71	0.095132	0.140048
72	0.100636	0.139548
73	0.092601	0.137506
74	0.097990	0.137499
75	0.090202	0.136351
76	0.095480	0.134942
77	0.087925	0.134924
78	0.093097	0.132880
79	0.085760	0.132950
80	0.090829	0.131852
81	0.083700	0.130850
82	0.088671	0.130158
83	0.093534	0.129595
84	0.086613	0.129081
85	0.091386	0.128250
86	0.084649	0.127232
87	0.089334	0.125788
88	0.082772	0.124589
89	0.087373	0.123370
90	0.080977	0.123023
91	0.085497	0.122260
92	0.079258	0.121840
93	0.083700	0.121364
94	0.077612	0.119274
95	0.081977	0.119792
96	0.086259	0.118242
97	0.080324	0.118910
98	0.084536	0.117628
99	0.078736	0.117731
100	0.082880	0.116678
101	0.077211	0.115853
102	0.081288	0.115215
103	0.075743	0.114903
104	0.079756	0.114568
105	0.074331	0.114325
106	0.078281	0.113924
107	0.072971	0.113253
108	0.076860	0.112351
109	0.080685	0.112457
110	0.075490	0.111330
111	0.079258	0.111620
112	0.074168	0.110700
113	0.077881	0.111148
114	0.072892	0.110252
115	0.076551	0.108806
116	0.071658	0.108790
117	0.075266	0.107943
118	0.070467	0.107740
119	0.074024	0.106984
120	0.077527	0.106913
121	0.072822	0.106028
122	0.076277	0.105501
123	0.071658	0.105238
124	0.075067	0.104838
125	0.070532	0.104498
126	0.073896	0.103923
127	0.069440	0.103690
128	0.072760	0.103268
129	0.068382	0.102197
130	0.071658	0.102208
131	0.067355	0.101564
132	0.070590	0.101621
133	0.066359	0.100992
134	0.069553	0.099332
135	0.072704	0.099594
136	0.068547	0.098617
137	0.071659	0.098682
138	0.067569	0.097819
139	0.070643	0.097920
140	0.066618	0.097012
141	0.069656	0.096715
142	0.065694	0.096292
143	0.068696	0.096133
144	0.064796	0.095889
145	0.067762	0.095591
146	0.063922	0.094791
147	0.066853	0.094934
148	0.069748	0.094298
149	0.065969	0.094508
150	0.068831	0.093915
151	0.065107	0.093425
152	0.067938	0.093306
153	0.064268	0.092520
154	0.067067	0.092560
155	0.063450	0.092030
156	0.066219	0.091956
157	0.062653	0.091565
158	0.065392	0.091113
159	0.061876	0.090908
160	0.064585	0.090734
161	0.061117	0.090511
162	0.063798	0.090391
163	0.066448	0.090163
164	0.063030	0.090070
165	0.065652	0.089497
166	0.062280	0.089062
167	0.064875	0.088795
168	0.061547	0.088540
169	0.064117	0.088321
170	0.060832	0.087730
171	0.063375	0.087827
172	0.060133	0.087372
173	0.062651	0.087504
174	0.059450	0.087064
175	0.061943	0.087191
176	0.058782	0.086881
177	0.061251	0.086688
178	0.058129	0.086486
179	0.060574	0.085821
180	0.057492	0.085924
181	0.059912	0.085510
182	0.062308	0.085298
183	0.059264	0.085079
184	0.061638	0.084857
185	0.058630	0.084699
186	0.060981	0.084547
187	0.058009	0.084097
188	0.060339	0.084101
189	0.057401	0.083819
190	0.059710	0.083813
191	0.056806	0.083509
192	0.059093	0.082663
193	0.056224	0.082866
194	0.058489	0.082216
195	0.055653	0.082422
196	0.057898	0.081852
197	0.060123	0.082070
198	0.057318	0.081524
199	0.059523	0.081480
200	0.056750	0.081147
201	0.058936	0.081144
202	0.056193	0.080783
203	0.058359	0.080787
204	0.055647	0.080278
205	0.057794	0.079837
206	0.055112	0.079642
207	0.057240	0.079381
208	0.054587	0.079285
209	0.056696	0.078830
210	0.054072	0.078805
211	0.056163	0.078421
212	0.053566	0.078421
213	0.055640	0.078057
214	0.057696	0.077828
215	0.055126	0.077536
216	0.057165	0.076796
217	0.054622	0.076862
218	0.056644	0.076362
219	0.054128	0.076270
220	0.056133	0.076013
221	0.053642	0.075905
222	0.055630	0.075648
223	0.053166	0.075585
224	0.055137	0.075384
225	0.052698	0.075356
226	0.054653	0.074943
227	0.052239	0.074963
228	0.054177	0.074667
229	0.051788	0.074112
230	0.053710	0.074082
231	0.055619	0.073579
232	0.053252	0.073624
233	0.055144	0.073326
234	0.052801	0.073351
235	0.054678	0.073050
236	0.052358	0.072898
237	0.054220	0.072696
238	0.051923	0.072621
239	0.053770	0.072457
240	0.051496	0.071928
241	0.053328	0.071878
242	0.051076	0.071525
243	0.052894	0.071469
244	0.050664	0.071209
245	0.052467	0.071178
246	0.054256	0.070872
247	0.052047	0.070893
248	0.053823	0.070617
249	0.051635	0.070650
250	0.053396	0.070469
251	0.051229	0.069956
252	0.052977	0.070039
253	0.050831	0.069681
254	0.052565	0.069760
255	0.050439	0.069432
256	0.052160	0.069447
257	0.050054	0.069211
258	0.051761	0.069237
259	0.049676	0.068991
260	0.051370	0.068529
261	0.049304	0.068459
262	0.050985	0.068224
263	0.048938	0.068185
264	0.050606	0.067920
265	0.048579	0.067935
266	0.050234	0.067638
267	0.051877	0.067618
268	0.049868	0.067418
269	0.051498	0.067454
270	0.049508	0.067187
271	0.051126	0.066804
272	0.049154	0.066828
273	0.050760	0.066545
274	0.048806	0.066566
275	0.050400	0.066357
276	0.048464	0.066312
277	0.050045	0.066182
278	0.048127	0.066163
279	0.049697	0.066002
280	0.047796	0.066000
281	0.049355	0.065783
282	0.047470	0.065491
283	0.049018	0.065503
284	0.047150	0.065322
285	0.048686	0.065306
286	0.046835	0.065083
287	0.048360	0.065073
288	0.049873	0.064899
289	0.048040	0.064862
290	0.049541	0.064715
291	0.047724	0.064390
292	0.049215	0.064441
293	0.047414	0.064186
294	0.048894	0.064243
295	0.047109	0.064007
296	0.048578	0.064033
297	0.046808	0.063819
298	0.048268	0.063817
299	0.046513	0.063684
}\datatable

\section{Experiments}
\label{app:experiments}

\begin{wrapfigure}[7]{r}{5cm}
\vspace{-0.5cm}
\small
\begin{tikzpicture}
\begin{axis}[height=4.3cm,
		width=5cm,
		transpose legend,
		grid=major,
		xlabel=Time,
		ylabel=Hellinger Error,
		xmin=10,
		xmax=100,
		ymax=0.6,
		thick,
    compat=newest,
    ylabel shift=-0.1cm,
		xlabel shift=-0.1cm,
		every axis legend/.append style={nodes={right},font=\tiny}]
\addplot[mark=none,black,smooth,domain=10:100] {sqrt(1/(2*x) * ln(20*x*(x+1))};
\addlegendentry{$g_t$}
\addplot[mark=none,dashed,black,smooth,domain=10:100] {sqrt(1/(2*x) * ln(20)};
\addlegendentry{$f_t$}
\addplot[mark=none,red,smooth] table[x index=0,y index=2] from \datatable;
\addlegendentry{$\hat h_t$}
\addplot[mark=none,blue,smooth] table[x index=0,y index=1] from \datatable;
\addlegendentry{$\hat q_t$}
\end{axis}
\end{tikzpicture}
\end{wrapfigure}
We set $\delta = \sfrac{1}{10}$ and $\M = \set{\nu_0, \cdots, \nu_{40}}$ where $\nu_k$ is the Bernoulli measure with parameter 
$\theta_k \defined k / 40$. We then sampled 20,000 sequences of length 100 from the Lebesgue measure $\mu = \nu_{20}$
and computed the average value of $\hat h_t$. For each $t$ we computed the estimated quantile $\hat q_t$ as the value such that 
$\hellinger{t}{\mu}{\xi} < \hat q_t$ for 90\% of the samples. We compare to 
\eq{
f_t &\defined \sqrt{{1 \over 2t} \ln{2 \over \delta}} &
g_t &\defined \sqrt{{1 \over 2t} \ln{2t(t+1) \over \delta}}
}
which are obtained from the Hoeffding and union bounds and satisfy
\eq{
\mu\left( \left|\hat \theta_t - \sfrac{1}{2}\right| \leq f_t\right) &\geq 1 - \delta &
\mu\left( \forall t: \left|\hat \theta_t - \sfrac{1}{2}\right| \leq g_t\right) &\geq 1- \delta
}
where $\hat \theta_t$ is the empiric estimator of parameter $\theta$.
Some remarks:
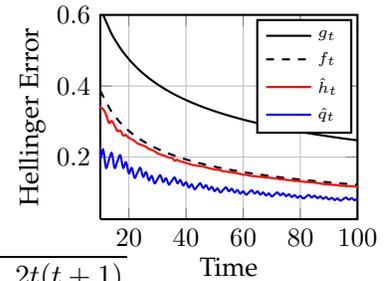
\begin{itemize}
\item The frequentist estimator $\hat \theta_t(\omega_{<t}) \approx \xi(1|\omega_{<t})$ is very tight with high probability. Therefore
comparing error between $\hat \theta_t$ and the true parameter $\sfrac{1}{2}$ is essentially the same as comparing $\xi(\cdot|\omega_{<t})$ and $\mu(\cdot|\omega_{<t})$.
\item The comparison to $f_t$ is not entirely fair to $\hat h_t$ for two reasons. First because $\hat h_t$ upper 
bounds $h_t$ with high probability
for all $t$ while $f_t$ does so only for each $t$ and secondly because $f_t$ was based on the total variation distance, which is smaller than the Hellinger distance.
\item The comparison between $\hat h_t$ and $g_t$ is not fair to $g_t$ because the application of the union bound was rather weak.
\item The comparison to the quantile is not entirely fair to $\hat h_t$, since $\hat q_t$ is computed for a single $\theta$ and individually for each $t$, while
$\hat h_t$ must work for all models in $\M$ and all $t$.
\item We also ran the experiment with 21 uniformly distributed environments with almost identical results showing that $\hat h_t$ is an excellent bound and strengthening our
belief that at least in this simple setting the bound of Theorem \ref{thm:Hellinger-estimator-convergence} can be substantially improved
in the i.i.d.\ case.
\item The results indicate that $\hat h_t$ tracks close to $f_t$ and $\hat q_t$, which essentially lower-bounds the optimum. We expect the definition of $g_t$ can be
improved to follow close to $\hat h_t$ and $f_t$ without weakening the bound (holding for all $t$), but doubt that anything does much better than $\hat h_t$.
\end{itemize}
\fi

%%%%%%%%%%%%%%%%%%%%%%%%%%%%%%%%%%%%%%%%%%%%%%%%%%%%%%%%%%%%%%%
% TABLE OF NOTATION
%%%%%%%%%%%%%%%%%%%%%%%%%%%%%%%%%%%%%%%%%%%%%%%%%%%%%%%%%%%%%%%
\iftecrep
\section{Table of Notation}\label{app:not}
\noindent
\hspace{-0.3cm}
\begin{tabu}{p{2cm} p{11cm}}
$\N$ & natural numbers \\
$\ind{expr}$ & indicator function of expression $expr$ \\
$\ln$ & natural logarithm  \\
$\A$  & finite or countable alphabet \\
$\A^*$& set of finite strings on $\A$ \\
$\A^\infty$& set of infinite strings on $\A$ \\
$x$, $y$ & symbols or strings in $X^*$ \\
$\ell(x)$ & length of string $x$ \\
$\epsilon$ & empty string of length $0$ \\
$\Gamma_x$ & cylinder set of $x$, $\Gamma_x \defined \set{x\omega : \omega \in X^\infty}$ \\
$\F_{<t}$ & sigma algebra generated by cylinders on strings of length $t-1$ \\
$\F$ & sigma algebra generated by by cylinders on strings of all finite lengths \\
$\M$  & environment class of hypothesis measures \\
$\mu$ & true measure from which sequences are sampled \\
$\xi$ & bayes mixture over all $\nu \in \M$ with prior $w:\M \to(0,1]$ \\
$\nu$, $\rho$ & measures in $\M$ \\
$\rho_t$ & random variable defined by $\rho_t(\omega) \defined \rho(\omega_t|\omega_{<t})$ \\
$\rho_{<t}$ & random variable defined by $\rho_{<t}(\omega) \defined \rho(\omega_{<t})$ \\ 
$\E_\mu$ & expectation w.r.t. $\mu$ \\
$w$   & prior distribution $w:\M \to (0,1]$ \\
$w_\nu$& prior weight of measure $\nu \in \M$ \\
$w_\nu(x)$ & posterior of measure $\nu \in \M$ having observed $x$ \\
$\ent$ & shannons entropy function \\
$\mu \absolute \xi$ & $\mu$ is absolutely continuous w.r.t. $\xi$ \\
$\KL{x}{\mu}{\xi}$ & KL divergence between predictive distributions of $\mu$ and $\xi$ given finite sequence $x$ \\
$\hellinger{x}{\mu}{\xi}$ & squared Hellinger distance between predictive distributions of $\mu$ and $\xi$ given finite sequence $x$ \\
$h_t$ & random variable $h_t(\omega) \defined \hellinger{\omega_{<t}}{\mu}{\xi}$ \\
$d_t$ & random variable $d_t(\omega) \defined \KL{\omega_{<t}}{\mu}{\xi}$ \\
$c_t$ & random variable $c_t(\omega) \defined \sum_{\nu\in\M} w_\nu(\omega_{<t}) \KL{\omega_{<t}}{\nu}{\xi}$ \\
$H_t$, $D_t$, $C_t$ & $\sum_{t=1}^\infty h_t$ and $\sum_{t=1}^\infty d_t$ and $\sum_{t=1}^\infty c_t$ respectively \\
$K$ & number of models in $\M$ \\
$\hat h_t$ & upper confidence bound on $h_t$ \\
$\hat d_t$ & upper confidence bound on $d_t$ \\
$\varepsilon, \delta$ & small positive reals with $\varepsilon$ typically an accuracy parameter and $\delta$ a confidence (probability) \\
\end{tabu}
\fi

\end{document}